\documentclass[11pt]{article}

\usepackage[numbers]{natbib}

\usepackage[utf8]{inputenc} 
\usepackage[T1]{fontenc}    
\usepackage{hyperref}       
\usepackage{url}            
\usepackage{booktabs}       
\usepackage{amsfonts}       
\usepackage{nicefrac}       
\usepackage{microtype}      
\usepackage{fullpage}
\usepackage{amsmath,amssymb,amsfonts,amsthm}
\usepackage[linesnumbered,ruled,vlined,noend]{algorithm2e}
\usepackage{color}
\usepackage{graphicx}
\usepackage{enumitem}
\usepackage{appendix}
\usepackage{parskip}

\usepackage[textsize=scriptsize,textwidth=3cm]{todonotes}

\usepackage[labelfont=bf,font=footnotesize]{caption}

\begingroup
    \makeatletter
    \@for\theoremstyle:=definition,remark,plain\do{%
        \expandafter\g@addto@macro\csname th@\theoremstyle\endcsname{%
            \addtolength\thm@preskip\parskip
            }%
        }
\endgroup

\DeclareMathOperator{\sgn}{sgn}

\SetKwProg{Fn}{Function}{:}{end}
\SetKw{KwCont}{continue}
\SetKw{KwIf}{if}
\SetKw{KwElse}{else}
\SetKw{KwThrow}{throw}
\SetKwInOut{Parameter}{Parameters}

\makeatletter
\newtheorem*{rep@theorem}{\rep@title}
\newcommand{\newreptheorem}[2]{%
\newenvironment{rep#1}[1]{%
 \def\rep@title{#2 \ref{##1}}%
 \begin{rep@theorem}}%
 {\end{rep@theorem}}}
\makeatother

\newreptheorem{theorem}{Theorem}

\newreptheorem{lemma}{Lemma}

\newtheorem{lem}{Lemma}

\newtheorem{thm}{Theorem}

\makeatother

\newcommand{\indim}{d} 
\newcommand{\hdim}{h} 
\newcommand{\binlim}{l} 
\newcommand{\binstep}{\epsilon} 

\title{Model Reconstruction from Model Explanations}
\author{
Smitha Milli \, Ludwig Schmidt \, Anca D. Dragan \, Moritz Hardt
\\
University of California, Berkeley \\
\texttt{\{smilli,ludwig,anca,hardt\}@berkeley.edu}
}
\date{}

\begin{document}

\maketitle

\begin{abstract}
We show through theory and experiment that gradient-based explanations of a model quickly reveal the model itself. Our results speak to a tension between the desire to keep a proprietary model secret and the ability to offer model explanations.

On the theoretical side, we give an algorithm that provably learns a two-layer ReLU network in a setting where the algorithm may query the gradient of the model with respect to chosen inputs. The number of queries is independent of the dimension and nearly optimal in its dependence on the model size. Of interest not only from a learning-theoretic perspective, this result highlights the power of gradients rather than labels as a learning primitive.

Complementing our theory, we give effective heuristics for reconstructing models from gradient explanations that are orders of magnitude more query-efficient than reconstruction attacks relying on prediction interfaces.
\end{abstract}

\section{Introduction}
Commercial machine learning models increasingly support consequential decisions in numerous domains including medical diagnosis, employment, and criminal justice. In such applications, there is now growing demand for methods that explain a model's decision. The secrecy of a model strongly fuels this demand.

At the same time, there are a number of valid reasons a company might wish to keep its machine learning models secret. The competitive value of the product is one consideration. Revealed models may also be easier to \emph{game}, resulting in diminished predictive power~\citep{dalvi2004adversarial,hardt2016strategic}. Yet another reason is that the model might leak sensitive information about the data it was trained on~\citep{shokri2017membership,carlini2018secret}.

In this work, we point out a tension between keeping a model secret and explaining its decisions. We show that a popular class of existing methods to explain a model's decision quickly reveals the model itself in what is typically an undesired side effect. 

Numerous explanation methods have been proposed in an ongoing line of research. Among these methods, \emph{saliency maps} are a widespread technique to highlight characteristics of an input deemed relevant for the prediction of a model. The most basic saliency map is to compute the gradient of the model with respect to a chosen input \citep{baehrens2010explain,simonyan2013deep} and numerous variants add different transformations to the raw gradients leading to some disagreement over which of these heuristics is preferable in what
context \citep{zeiler2014visualizing, springenberg2014striving, smilkov2017smoothgrad, sundararajan2017axiomatic}. Abstracting away from these implementation details, we focus on reconstructing models given the basic underlying primitive, which is gradients of the model with respect to its inputs.

\subsection{Our contributions}
Our contributions are twofold, spanning both a theoretical and experimental component.

\textbf{Learning from input gradients.} On the theoretical side, we introduce a model of learning from \emph{input gradient queries}. In this model, a learning algorithm can observe gradients of an unknown model at chosen query inputs. This model turns out to be rich in its mathematical structure and connections to standard learning models, such as learning from \emph{membership queries}, in which the learner can request the model's prediction at a given input.

In our setting, since the gradient provides more information than a single label, there is hope that learning algorithms can get by with far fewer queries. We prove that this is indeed the case. To build up intuition with a simple example, consider a linear model $f(x)= \langle w, x\rangle,$ specified by a weight vector $w\in\mathbb{R}^d.$ The gradient of the model with respect to any input $x$ is just equal to the model parameters $w=\nabla_x f(x).$ Thus, we can learn a linear model from a single input gradient query. 

Going beyond linear models,
We analyze two-layer neural networks with ReLU transitions of the form $f(x)=\langle w, \mathrm{ReLU}(Ax)\rangle$ where $A \in \mathbb{R}^{h \times d}$. Here, $\mathrm{ReLU}(u)=\max\{u,0\}$ applies coordinate-wise to a vector. The problem of learning such networks has received much renewed interest in the last few years as it poses a non-trivial challenge en route to deeper non-linear models \cite{safran2017spurious,tian2017analytical,zhong2017recovery}.

\begin{thm}[informal]
\label{thm:main-thm}
Assuming the rows of the weight matrix~$A$ are linearly independent, our algorithm recovers a functionally equivalent model from $O(h\log h)$ input gradient queries and function evaluations with high probability.
\end{thm}
The $O(h \log h)$ queries our theorem requires is optimal to within a logarithmic factor, since it takes $dh+h$ parameters to specify the model, and each query reveals only $O(d)$ numbers. Furthermore, compared to membership queries, gradient queries reduce the number of queries needed by approximately a factor of $d$, since it takes $\Omega(dh)$ membership queries to specify the model. 

Although our algorithm enjoys an intuitive geometric interpretation, the proof requires a delicate argument, as well as an anti-concentration bound that may be useful independently.

\textbf{Practical reconstruction methods.} 
In a second step, we explore practically effective heuristics to reconstruct a model from input gradient queries. Our experiments show that reconstructing models from explanations is not just a theoretical concern. If a company were to provide an \emph{explanation API} with standard saliency maps, it would effectively give up the underlying model, which it may not be willing to do for reasons mentioned above. This situation parallels an ongoing investigation on \emph{stealing models from prediction APIs}~\citep{tramer2016stealing}. However, as our results show, with explanation APIs we need far fewer queries, thus greatly exacerbating the threat of model leakage.

Our experiments focus on a heuristic for learning from input query gradients. While our theoretical method is specific to two-layer networks, our heuristic is agnostic to the shape of the target model. At the outset, our heuristic simply queries a number of input gradients and fits a model against the observed gradients in much the same way we would fit a model against labels. We find that this heuristic reduces the number of queries needed to learn models on MNIST and CIFAR10 by orders of magnitude, even in cases where the model class is unknown or the data distribution is unknown.

\textbf{Conclusion.} Our work demonstrates that establishing usable explanation methods for machine learning models faces another hurdle in commercial applications. Whatever criteria of explanation quality we choose must be weighed against the risk of model leakage resulting from the method at hand. We see our work as only a first step in this new direction that raises many intriguing questions.


Does our theoretical result extend to depth-$3$ networks? Ignoring computational efficiency, what is the optimal query complexity? In particular, can we learn a $k$-layer ReLU network with $h$ units at each layer from only $\tilde O(kh)$ queries? Can we design useful explanation methods resilient to model reconstruction attacks? Although a natural and important question to ask, there is no currently agreed upon measure of explanation quality, which makes it difficult to formally study this trade-off. 


%

\section{Problem statement: reconstructing a two-layer ReLU network} \label{sec:model}
We consider the problem of finding a classifier $\hat{f}$ identical to an unknown classifier $f$ when given access to membership and gradient queries. That is, we assume access to an oracle that given a query input $x$ returns the evaluation of $f$ at $x$ and the gradient $\nabla_x f(x)$ of $f$ with respect to $x.$

We analyze the case where the function $f : \mathbb{R}^{d} \rightarrow \mathbb{R}$ is represented by a one hidden-layer neural network with ReLU activations:
\begin{align}\label{eq:two-layer-form}
    f(x) = \sum_{i=1}^{h} w_i \max(A_i^\top x, 0)\,.
\end{align}
Here, the model parameters are $A \in \mathbb{R}^{h \times \indim}$ and $w \in \mathbb{R}^{h}$. We use $A_i$ to denote the $i$-th row of $A$. We make the following three assumptions:
\begin{enumerate}[noitemsep, nolistsep]
    \item The rows $A_1, \dots A_h$ are unit vectors. \label{as:unit-norm}
    \item No two rows $A_i$ and $A_j$ with $i\ne j$ are collinear, i.e., $\langle A_i, A_j \rangle \leq 1-c$ for some $c>0$. \label{as:collinearity}
    \item The rows $A_1, \dots, A_h$ are linearly independent.
\end{enumerate}

The first two assumptions are without loss of generality, 
as they follow from simple reparameterizations of the network that involve scaling $w$ or $A$ or reducing the hidden dimension.

Our main result is the following theorem, which shows that our sample complexity for learning the function with gradient queries has no dependence on the input dimension $d$.

\begin{reptheorem}{thm:main-thm}
Suppose, the unknown function~$f$ satisfies our assumptions. Then,
with probability $1 - \delta$, Algorithm \ref{algo:model-recovery} succeeds to find a function $\hat{f}$ such that $\hat{f} = f$ in $O(h \log \frac{h}{\delta})$ queries. If the Algorithm fails, then it notifies of the failure.
\end{reptheorem}

Section \ref{sec:algo} contains our algorithm and proof of correctness. In Appendix \ref{app:membership} we show that our algorithm can also be converted to one which learns the function $f$ in $O(dh \log \frac{h}{\delta})$ membership queries by using membership queries to approximate gradients of $f$.

\section{Algorithm} \label{sec:algo} \label{sec:gradients}
Before we formally introduce our algorithm, we briefly provide some high-level intuition.
First, note that we can express our two-layer ReLU networks as
\begin{align}
    \textstyle f(x) = \sum_{i=1}^{h} g(x)_iw_iA_i^\top x\,,
\end{align}
where $g(x) = \mathbb{I}\{Ax \geq 0 \}$. The \textit{separating hyperplanes} defined by the normal vectors $A_1, \dots A_h$ split the input space into cells represented by the possible values of $g(x)$. Within each such cell, the function $f$ is linear. See Figure \ref{fig:algo-z-recovery} for an example visualization of these cells.

Our algorithm can be separated into two steps. First, we find the separating hyperplanes of $f$. In particular, we recover unsigned, weighted normal vectors $w_iA_i$ or $-w_iA_i$ for $i \in [h]$. The second step then recovers the sign information for these normal vectors. More precisely, the two steps are the following:
\begin{enumerate}
    \item Recover a matrix $Z \in \mathbb{R}^{h \times d}$ such that $Z_{p(i)} = w_iA_i$ or $Z_{p(i)} = -w_i A_i$ for some permutation $p$ of $[h]$. (Algorithm \ref{algo:z-recovery})
    \item Recover a vector $s \in \{-1, 0, 1\}^{2h}$ such that $f(x) = \begin{bmatrix}\max(Zx, 0)^\top & \max(-Zx, 0)^\top \end{bmatrix}s$. (Algorithm \ref{algo:s-recovery})
\end{enumerate}

Together, the matrix $Z$ and vector $s$ identify the function $f$. We analyze the first step in Section \ref{sec:recover-z} and the second step in Section \ref{sec:recover-s}.

\begin{algorithm}[h]
\SetAlgoLined
\Fn{learnModel($h$, $\epsilon$, $l$)}{
 $Z \gets$ \textit{recoverZ}($h$, $\epsilon$, $l$) \\
 $s \gets$ \textit{recoverS}($Z$) \\
 \KwRet\ $Z, s$}
\caption{Recovery of $f$}
\label{algo:model-recovery}
\end{algorithm}

\setcounter{algocf}{0}
\renewcommand\thealgocf{1\alph{algocf}}

\begin{algorithm}[h]
\SetAlgoLined
 \Fn{recoverZ($h$, $\epsilon$, $l$)}{
    Pick $u, v \sim \mathcal{N}(0, I_d)$ and let $Z \in \mathbb{R}^{h \times d}$ \\
    $t_l$, $t_r$ $\gets -l, \,l$ \\
    \For{$i = 1, \dots h$}{
        $Z_i$, $t_l \gets$ binarySearch($t_l$, $t_r$, $\epsilon$) \\
    }
    \KwRet $Z$
 }

 \Fn{binarySearch($t_l$, $t_r$, $\epsilon$)}{
    \While{$t_l \leq t_r$}{
        $t_m \gets (t_l+t_r)/2$ \\
        $x_l \gets u + t_lv$, ~$\,x_m \gets u+t_mv$, ~$x_r \gets u + t_rv$  \\
        \uIf{$t_r - t_l \leq \epsilon$}{
            \KwRet $\nabla f(x_r) -\nabla f(x_l), ~t_r$ \\
        }
        \uIf{$\|\nabla f(x_l) -\nabla f(x_m)\|_2 > 0$}{
            $t_r \gets t_m$
        }\uElseIf{$\|\nabla f(x_m) -\nabla f(x_r)\|_2 > 0$}{
            $t_l \gets t_m$
        }
        \KwThrow Failure
    }
    \KwThrow Failure
    
 }
 \caption{Recovery of $Z$}
 \label{algo:z-recovery}
\end{algorithm}

\begin{algorithm}[h]
\SetAlgoLined
  \Fn{recoverS(Z)}{
    Pick $X \in \mathbb{R}^{d \times h}$ such that $\nabla f(x_1) = \dots = \nabla f(x_h)$ and $\text{Rank}(ZX) = h$. (See Appendix \ref{sec:pick-X})\\
    $M \gets \begin{bmatrix} \max(ZX, 0)^\top & \max(-ZX, 0)^\top \\ \max(-ZX, 0)^\top & \max(ZX, 0)^\top \end{bmatrix}$ \\
    Solve for $s \in \mathbb{R}^{2h}$ such that $Ms = [f(x_1), \dots f(x_h), f(-x_1), \dots f(-x_h)]$ \\
    \KwRet $s$
 }
\caption{Recovery of $s$}
\label{algo:s-recovery}
\end{algorithm}

\subsection{Step one: recovering the separating hyperplanes} \label{sec:recover-z}
\begin{figure}
    \centering
    \includegraphics[scale=0.5]{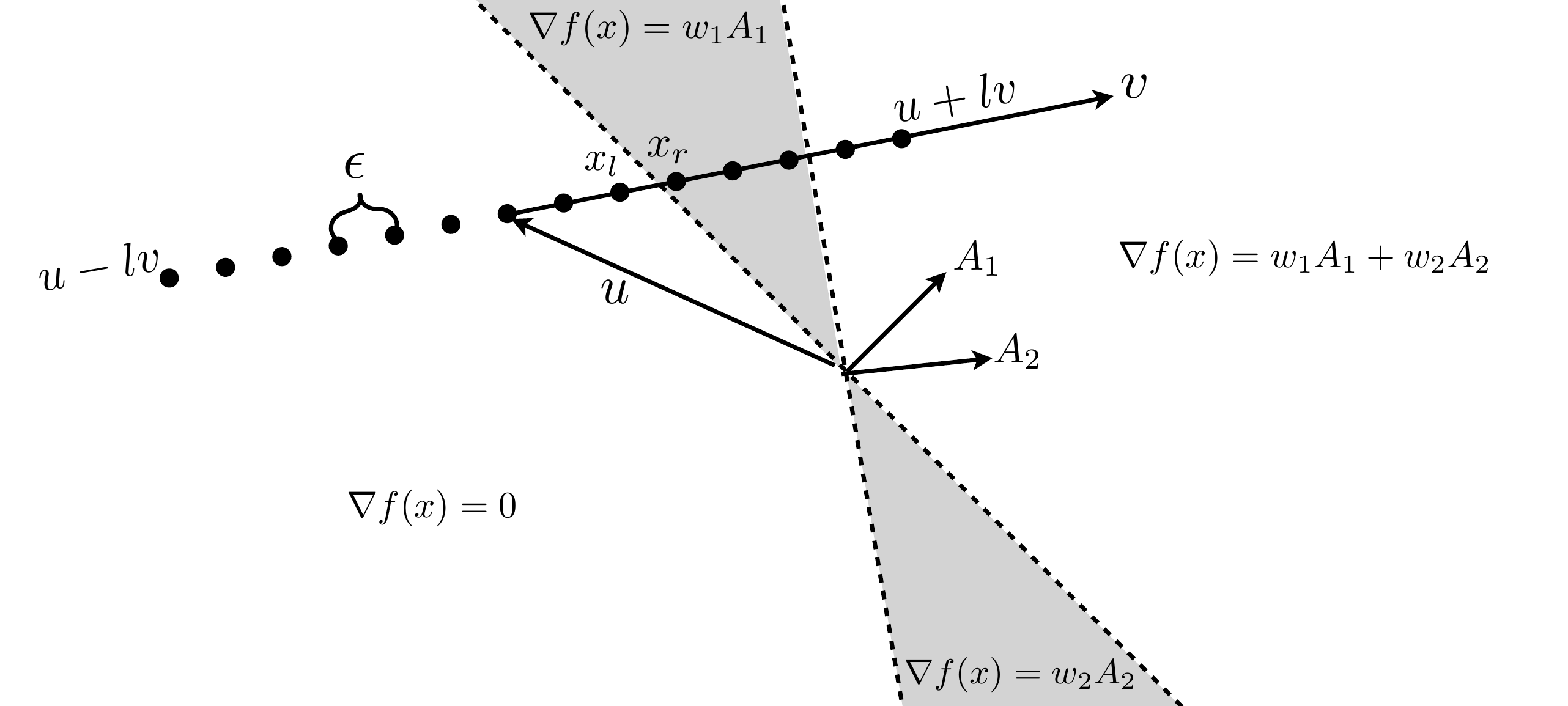}
    \caption{An illustration of Algorithm \ref{algo:z-recovery} when the input domain of the function $f$ is $\mathbb{R}^2$ and the hidden dimension $h$ is equal to two. The two hyperplanes with normal vector $A_1$ and $A_2$ separate the input space into four cells where the gradient of $f$ is constant.  Algorithm \ref{algo:z-recovery} picks two random vectors $u$ and $v$ and searches for a change in the gradient of $f$ using a binary search along a line segment between $u$ and $v$. When two points are found that are sufficiently close, but have differing gradients, then the difference in their gradients is added as a row to the recovered matrix $Z$. For example, $\nabla f(x_r) - \nabla f(x_l) = w_1A_1$ is added to $Z$. By running the binary search $h$ times, Algorithm \ref{algo:z-recovery} recovers $w_iA_i$ up to a sign for all $i \in [h]$.}
    \label{fig:algo-z-recovery}
\end{figure}


Algorithm \ref{algo:z-recovery} finds the separating hyperplanes by exploiting the structure of the gradient of $f$:
\begin{align*}
\textstyle \nabla f(x) = \sum_{i=1}^{h} g(x)_i w_iA_i\,,
\end{align*}
where $g(x) = \mathbb{I}\{Ax \geq 0 \}$ as before. Note that points within the same cell have the same gradient.
So if we find two points $x$ and $y$ with different gradients, we know at least one separating hyperplane must be between $x$ and $y$. Moreover, if the points $x$ and $y$ are sufficiently close to each other, then it is likely that there is only one separating hyperplane between them. In that case, we can then use the difference of gradients to recover a hyperplane (up to signs).
This is because each gradient is simply a sum of a subset of $\{w_iA_i\}_{i=1}^{h}$, and so the difference $\nabla f(y) - \nabla f(x)$ is equal to either $w_iA_i$ or $-w_iA_i$ for some $i \in [h]$. 

In this way, Algorithm \ref{algo:z-recovery} isolates changes in the gradient of $f$ to recover $w_iA_i$ up to a sign for every $i \in [h]$. Figure \ref{fig:algo-z-recovery} provides an illustrated explanation of the algorithm, which we briefly sketch below:
\begin{enumerate}
    \item Pick $u, v \sim \mathcal{N}(0, I_d)$.
    \item Run a binary search with resolution $\epsilon$ along a portion of the line segment between $u - lv$ and $u + lv$ for some $l \in \mathbb{R}$ to find two points $x_l$ and $x_r$ that are sufficiently close ($\|x_r - x_l\|_2 \leq \epsilon \|v\|_2$), but have differing gradients. Add $\nabla f(x_r) - \nabla f(x_l)$ as a row to the matrix $Z$. With high probability, $\nabla f(x_r) - \nabla f(x_l)$ is equal to $w_iA_i$ for some $i \in [h]$.
    \item Repeat Step (2) $h$ times to recover all rows $w_iA_i$ up to their sign, which become the rows of the matrix $Z$.
\end{enumerate}

The proof of correctness relies on showing that with high probability, the following two events hold: (i) The points at which the gradient of $f$ changes are spaced sufficiently far apart. (ii) The same gradient change points are within some line segment of $u$ and $v$ that is not too big. The change points can then be found with a binary search that is bounded within a range that is not too large and uses step sizes that are not too small. In the next lemma, we prove correctness of the binary search given that the change points are spaced appropriately.

\begin{lem}
\label{lem:det-zrecovery}
Let $u, v \in \mathbb{R}^{\indim}$ be such that $\langle A_i, v \rangle \neq 0$ for all $i \in [h]$.
For each $i \in [h]$, also let $t_i \in \mathbb{R}$ be such that $\langle A_i, u + t_i v \rangle = 0$.
If for all $i, j \neq i$ we have $|t_i - t_j| \geq \epsilon$ and $|t_i| \leq \binlim$, then Algorithm \ref{algo:z-recovery} returns a matrix $Z \in \mathbb{R}^{h \times d}$ such that $Z_{p(i)} = w_iA_i$ or $Z_{p(i)} = -w_i A_i$ for some permutation $p$ of $[h]$.
\end{lem}
\begin{proof}
Let $k_1, \dots, k_h$ be the indices such that $t_{k_1} < t_{k_2} < \dots < t_{k_h}$. To prove the lemma we will show that on the $i$-th call to $\textit{binarySearch}$, either $-w_{k_i}A_{k_i}$ or $w_{k_i}A_{k_i}$ is added as a row to matrix $Z$. 

First, we make the following assumption, which we will later prove: assume that $t_{k_i} = \min_{j : t_j \geq t^{(i)}_{l}} t_j$ where $t^{(i)}_l$ is the value of the variable $t_l$ at the start of the $i$-th call to \textit{binarySearch}. Given this assumption, the $i$-th call to \textit{binarySearch} adds $-w_{k_i}A_{k_i}$ or $w_{k_i}A_{k_i}$ to the matrix $Z$. To see this, note that on each iteration of the while loop in \textit{binarySearch} either the variable $t_l$ increases or the variable $t_r$ decreases, and thus \textit{binarySearch} always terminates. However, $t_l$ dose not increase past $t_{k_i}$ and $t_r$ does not decrease past $t_{k_i}$. So, when the condition for termination of the while loop is met we have $|t_l - t_r| \leq \epsilon$, $t_l \leq t_{k_i}$, and $t_r \geq t_{k_i}$. Since $|t_{k_{j}} - t_{k_{i}}| \geq \epsilon$ for all $j \neq i$, the row $\nabla f(t_r) - \nabla f(t_l)$ returned by \textit{binarySearch} is equal to either $w_{k_{i}}A_{k_{i}}$ or $-w_{k_{i}}A_{k_{i}}$.

Now we revisit the assumption that $t_{k_i} = \min_{j : t_j \geq t^{(i)}_{l}} t_j$. We prove the assumption by induction. The base case $i = 0$ is clearly true: $t_{k_1} = \min_{j} t_j = \min_{j : t_j \geq t^{(i)}_{l}} t_j$ because $t^{(1)}_l = -l$ and $-l \leq t_{k_1} < t_{k_2} < t_{k_l} \leq l$. On the $(i + 1)$-th call to $\textit{binarySearch}$ the variable $t_l$ is set to the value of $t_r$ when the $i$-th call to $\textit{binarySearch}$ terminated. When the $i$-th call to $\textit{binarySearch}$ finishes, the value of the variable $t_r$ is above $\min_{j : t_j \geq t^{(i)}_l} t_j = t_{k_i}$, but less than $t_{k_{i+1}}$. Thus, $t_{k_{i+1}} = \min_{j : t_j \geq t^{(i+1)}_l} t_j$.

Therefore, the returned matrix $Z$ is such that $Z_{p(i)} = w_iA_i$ or $Z_{p(i)} = -w_iA_i$ where the permutation $p$ of $[h]$ is defined by $p(i) = j$ where $k_j = i$.
\end{proof}

The next two lemmas (proved in Appendix \ref{app:step_one}) establish the necessary anti-concentration and concentration bounds for showing that the change points are spaced sufficiently far apart (Lemma \ref{lem:anti-conc-bin}), but still within some line segment of $u$ and $v$ that is not too big (Lemma \ref{lem:conc-bin}).

\begin{lem}
\label{lem:anti-conc-bin} Let $a, b \in \mathcal{S}^{d-1}$ be unit vectors such that $|\langle a, b \rangle| \leq 1 - c$ for some scalar $c \in [0, 1]$. Suppose we pick random vectors $u, v \sim \mathcal{N}(0, I_d)$. Let $t_1, t_2 \in \mathbb{R}$ be scalars such that $\langle a, u + t_1 v \rangle = 0$ and $\langle b, u + t_2 v \rangle = 0$.\footnote{With probability one such a $t$ exists.} Then,
\begin{align*}
   P(|t_1 - t_2| \leq \epsilon) \leq 3^{\frac{4}{3}}\left (\frac{\epsilon}{c} \right )^{\frac{2}{3}}.
\end{align*}
\end{lem}

\begin{lem}
\label{lem:conc-bin}
Let $a \in \mathcal{S}^{d-1}$ be a unit vector. Suppose we pick random vectors $u, v \sim \mathcal{N}(0, 1)$. Let $t \in \mathbb{R}$ be the value such that $\langle a, u + t v \rangle = 0$. Then,
\begin{align*}
   P(|t| \geq l) \leq \frac{2}{\pi l}.
\end{align*}
\end{lem}

Finally, the proof of our main theorem for Algorithm \ref{algo:z-recovery} follows by combining the probabilistic guarantees of Lemmas \ref{lem:anti-conc-bin} and  \ref{lem:conc-bin} with the deterministic proof of correctness in Lemma \ref{lem:det-zrecovery}.
\begin{thm}
\label{thm:prob-zrecovery}
With probability $1 - \delta$, Algorithm \ref{algo:z-recovery} succeeds in $O(h \log \frac{h}{\delta})$ queries. If the Algorithm succeeds, it returns a matrix $Z \in \mathbb{R}^{h \times d}$ such that $Z_{p(i)} = w_iA_i$ or $Z_{p(i)} = -w_i A_i$ for some permutation $p$ of $[h]$. If the Algorithm fails, then it notifies of the failure.
\end{thm}
\begin{proof}
By Lemma \ref{lem:det-zrecovery}, if $|t_i - t_j|$ and $|t_i| \leq \binlim$ for all $i$ and $j \neq i$, then Algorithm \ref{algo:z-recovery} succeeds. The probability of this event can be lower-bounded as the following.
\begin{align*}
    & P(\forall i, j \neq i : |t_i - t_j| \geq \binstep, \, |t_i| \leq \binlim) \\
    & \geq 1 - \sum_{i=1}^{\hdim}\sum_{j\neq i} P(|t_i - t_j| \leq \binstep) - \sum_{i=1}^{\hdim} P(|t_i| \geq \binlim) & \text{(Union bound)} \\
    & \geq 1 - 3^{\frac{4}{3}} \left ( \frac{\binstep}{c} \right )^{\frac{2}{3}}\hdim^2 - \sum_{i=1}^{\hdim} P(|t_i| \geq \binlim) & \text{(Lemma \ref{lem:anti-conc-bin})} \\
    & \geq 1 - 3^{\frac{4}{3}} \left ( \frac{\binstep}{c} \right )^{\frac{2}{3}}\hdim^2 - \frac{2}{\pi \binlim} \hdim & \text{(Lemma \ref{lem:conc-bin})}
\end{align*}
Let $\delta = 3^{\frac{4}{3}} \left ( \frac{\binstep}{c} \right )^{\frac{2}{3}}\hdim^2 - \frac{2}{\pi \binlim} \hdim$.
Set $l = h^2$. Then, solving for $\binstep$ yields $\epsilon = 3^{-2}c\frac{(\delta + \frac{2\pi}{h})^{\frac{3}{2}}}{h^3}$. So, Algorithm \ref{algo:z-recovery} succeeds with probability $1- \delta$ and uses less than $\hdim \log \left ( \frac{\binlim}{\epsilon} \right)$ queries, which is upper bounded as the following.
\begin{align*}
    \hdim \log \left ( \frac{\binlim}{\epsilon} \right) & = \hdim \log  \left ( \frac{h^2}{ 3^{-2}c\frac{(\delta + \frac{2\pi}{h})^{\frac{3}{2}}}{h^3}} \right ) \\
    & = 5 \hdim \log  \left ( \frac{h}{ 3^{-2}c (\delta + \frac{2\pi}{h})^{\frac{3}{2}}} \right ) \\
    & \leq O \left (h \log \frac{h}{\delta} \right )
\end{align*}
\end{proof}

\subsection{Step two: recovering the signs of the normal vectors} \label{sec:recover-s}
Algorithm \ref{algo:z-recovery} recovers unsigned, weighted normal vectors: $w_iA_i$ or $-w_iA_i$ for $i \in [h]$. But to identify the function $f$, we still need the sign of these vectors. In Algorithm \ref{algo:s-recovery}, we recover a vector $s \in \{-1, 0, 1\}^{2h}$ that encodes this sign information. Precisely, Algorithm \ref{algo:s-recovery} returns a vector $s$ such that 
\begin{align*}
    f(x) = \begin{bmatrix} \max(Zx, 0)^\top & \max(-Zx, 0)^\top \end{bmatrix}s \,.
\end{align*}
where
\begin{align*}
    s_i = \begin{cases}
        \sgn(w_i) & 1 \leq i \leq h,\, z_i = |w_i|A_i \\
        0 & h + 1 \leq i \leq 2h,\,  z_i = |w_i|A_i \\
        0 & 1 \leq i \leq h,\, z_i = -|w_i|A_i \\
        \sgn(w_i) & h + 1 \leq i \leq 2h,\, z_i = -|w_i|A_i
    \end{cases} \,.
\end{align*}

It is clear that if Algorithm \ref{algo:s-recovery} returns the vector $s$, then the function $f$ is identified. Algorithm \ref{algo:s-recovery} solves $2h$ linear equations to determine the vector $s$. To prove correctness of Algorithm \ref{algo:s-recovery}, we show that the $2h$ query points picked in the algorithm lead to a determined set of linear equations.

\begin{lem}
\label{lem:fullrank}
Let $Z \in \mathbb{R}^{h \times d}$ be a matrix such that $Z_{p(i)} = w_iA_i$ or $Z_{p(i)} = -w_iA_i$ for a permutation $p$ of $[h]$. Let $x_i$ denote the $i$-th column of a matrix $X \in \mathbb{R}^{d \times h}$. Suppose $\nabla f(x_1) = \dots = \nabla f(x_h)$, $(ZX)_{ij} \neq 0$, and $Rank(ZX) = h$ for all $i, j \in [h]$. Then, the $2h \times 2h$ matrix defined as
\begin{align}
    \label{lem:matrix}
    M = \begin{bmatrix}
        \max(ZX, 0)^\top & \max(-ZX, 0)^\top \\
        \max(-ZX, 0)^\top & \max(ZX, 0)^\top
    \end{bmatrix}
\end{align}
is full-rank.
\end{lem}
\begin{proof}
Since $\nabla f(x_1) = \dots = \nabla f(x_h)$ and $(ZX)_{ij} \neq 0$, we know $\mathbb{I}\{Zx_1 > 0 \} = \dots = \mathbb{I}\{Zx_h > 0 \}$, and that we could always negate rows of the matrix $Z$ so that $\mathbf{1} = \mathbb{I}\{Zx_1 > 0 \} = \dots = \mathbb{I}\{Zx_h > 0 \}$. Thus, we can assume without loss of generality that $(ZX)_{ij} > 0$ for all $i, j \in [h]$. Then, the matrix $M$ can be expressed as the following.
\begin{align*}
    M = \begin{bmatrix}
    (ZX)^\top & 0 \\
    0 & (ZX)^\top
    \end{bmatrix}
\end{align*}
The determinant of the matrix is $\det(M) = \det((ZX)^2 - 0) = \det^2(ZX) > 0$. Thus, $M$ is a full-rank matrix.
\end{proof}

In Appendix \ref{sec:pick-X} we describe a simple linear program that can be used to pick a matrix $X$ that satisfies the conditions of the above Lemma \ref{lem:fullrank}. Since Algorithm \ref{algo:s-recovery} picks such a matrix $X$, Lemma \ref{lem:fullrank} immediately implies our main theorem proving correctness of Algorithm \ref{algo:s-recovery}.

\begin{thm} 
\label{thm:s-recovery}
If Algorithm \ref{algo:s-recovery} is given a matrix $Z \in \mathbb{R}^{h \times d}$ such that $Z_{p(i)} = w_iA_i$ or $Z_{p(i)} = -w_iA_i$ for a permutation $p$ of $[h]$, then it returns a vector $s \in \{-1, 0, 1\}^{2h}$ such that the function $f$ is equal to $f(x) = \begin{bmatrix} \max(Zx, 0)^\top & \max(-Zx, 0)^\top \end{bmatrix}s$.
\end{thm}
\begin{proof}
Algorithm \ref{algo:s-recovery} uses $2h$ queries to construct a $X \in \mathbb{R}^{2h \times 2h}$ that satisfies the conditions of Lemma \ref{lem:fullrank}. Thus, the resulting set of $2h$ linear equations are determined and Algorithm \ref{algo:s-recovery} returns the unique vector $s$ corresponding to its solution.
\end{proof}

Together, Theorem \ref{thm:prob-zrecovery} proving correctness of Algorithm \ref{algo:z-recovery} and Theorem \ref{thm:s-recovery} proving correctness of \ref{algo:s-recovery} imply our main Theorem \ref{thm:main-thm} that proves correctness of Algorithm \ref{algo:model-recovery}.

\begin{reptheorem}{thm:main-thm}
Suppose the unknown function~$f$ satisfies the assumptions in Section \ref{sec:model}. Then,
with probability $1 - \delta$, Algorithm \ref{algo:model-recovery} succeeds to find a function $\hat{f}$ such that $\hat{f} = f$ in $O(h \log \frac{h}{\delta})$ queries. If the Algorithm fails, then it notifies of the failure.
\end{reptheorem}
\begin{proof}
By Theorem \ref{thm:prob-zrecovery}, with probability $1-\delta$, Algorithm \ref{algo:z-recovery} returns a matrix $Z$ that satisfies the conditions of Theorem \ref{thm:s-recovery} in $O(h\log \frac{h}{\delta})$ queries. By Theorem \ref{thm:s-recovery}, Algorithm \ref{algo:s-recovery} then returns a vector $s$ such that $f(x) = \begin{bmatrix} \max(Zx, 0)^\top & \max(-Zx, 0)^\top \end{bmatrix}s$ in $O(h)$ queries. Thus, overall Algorithm \ref{algo:model-recovery} succeeds with probability $1-\delta$ in $O(h \log h)$ queries.
 \end{proof}


\section{Experimental design} \label{sec:exps}
While our theoretical analysis provides insight into the power of gradient queries over membership queries, it is specific to a two-layer ReLU network. To complement our theory, we also experimentally investigate the impact of gradients on reconstructing models used in practice.

In order to compare to reconstructing with membership queries alone, our method for learning with gradients is a modification of a simple heuristic used to reconstruct models from membership queries: training a new classifier $\hat{f}$ to match the outputs of $f$ ~\citep{tramer2016stealing,papernot2017practical}. When we have access to gradients we can also train the classifier $\hat{f}$ to match the gradients of $f$ by minimizing a loss on the gradients: $\ell_{G}(x) = \|\nabla f(x) - \nabla \hat{f}(x) \|^2_2$. Furthermore, we can trade off between the gradient loss $\ell_{G}$ with a loss on the membership queries $\lambda \ell_{M}$ to create a joint loss $\ell_{J}(x) = \ell_{G}(x) + \lambda \ell_{M}(x)$.

We test how gradient queries help by measuring the accuracy of $\hat{f}$ when trained using $\ell_{J}(x)$ versus when trained only on the membership query loss, $\ell_{M}(x)$. In our experiments $\ell_{M}(x)$ is the cross-entropy loss between $f(x)$ and $\hat{f}(x)$.  Next, we describe our experimental design in detail.

\textbf{Manipulated factors.} We manipulate three independent variables. First, we manipulate the \textit{type of query}. We test membership only queries as well as membership and gradients. Further, because in practice explanations often provide a processed version of the gradients, instead of the raw gradients, we also test membership and gradients processed with SmoothGrad, a saliency map denoising technique \citep{smilkov2017smoothgrad}. Instead of returning the raw gradient $\nabla f(x)$, SmoothGrad returns an average of gradients around the input $x$: $\widetilde{\nabla} f(x) = \sum_{i=1}^{N} \frac{1}{N} \nabla f(x + z_i)$ where $z_i \sim \mathcal{N}(0, \sigma I)$. 

Second, we manipulate the \textit{complexity of the task} to test whether gradients help more or less on more complex tasks. We experiment on both MNIST and CIFAR10.
Finally, we manipulate the \textit{complexity of the model class} to test whether gradients help more when the model is simpler. We train three models on each of the two tasks that are chosen to display a range of complexity.

\textbf{Dependent measure.} We measure the accuracy of our reconstructed classifier $\hat{f}$ on a test set of 10,000 images from the task (MNIST or CIFAR10).

\textbf{Experimental procedure.} We split our datasets into three parts:
\begin{itemize}[noitemsep,nolistsep]
    \item A training set of images and ground-truth labels for the true classifier $f$. The training set for MNIST has 50,000 examples and for CIFAR10 has 40,000 examples.
    \item A training set of 10,000 images for the reconstructed classifier $\hat{f}$. Note that $\hat{f}$ does not have access to ground-truth labels, so it must query $f$ for labels.
    \item A test set of 10,000 images and ground-truth labels for $f$ and $\hat{f}$.
\end{itemize}

We first train models to serve as the true classifier $f$. We train three types of models on MNIST: a 1-layer network (multinomial logistic regression), a 2-layer neural network with ReLu activations, and a network with two convolutional layers (each followed by a max-pool layer) followed by two dense layers. We also train three types of models on CIFAR10: the same convolutional network used for MNIST (with the input dimension changed appropriately), a VGG11 network \citep{simonyan2014very}, and a ResNet-18 network \citep{he2016deep}.

Next, we train a new classifier $\hat{f}$ from the same model class as the true classifier $f$. The inputs $x$ given to $\hat{f}$ are randomly sampled from the training set for $\hat{f}$. After training, we compute the accuracy of our reconstructed classifier $\hat{f}$ on the test set.

\textbf{Follow-up experiments: unknown model class and data distribution} An adversary trying to reconstruct the classifier $f$ may not know the model class of $f$ or the data distribution. So, in follow-up experiments we (1) reconstruct the classifier $f$ with a classifier $\hat{f}$ from a different model class and (2) reconstruct the classifier $f$ using Gaussian generated queries. In these follow-up experiments we analyze the same factors, but with a subset of conditions. 

\section{Experimental results and discussion}
\begin{figure} 
    \centering
    \includegraphics[scale=0.45]{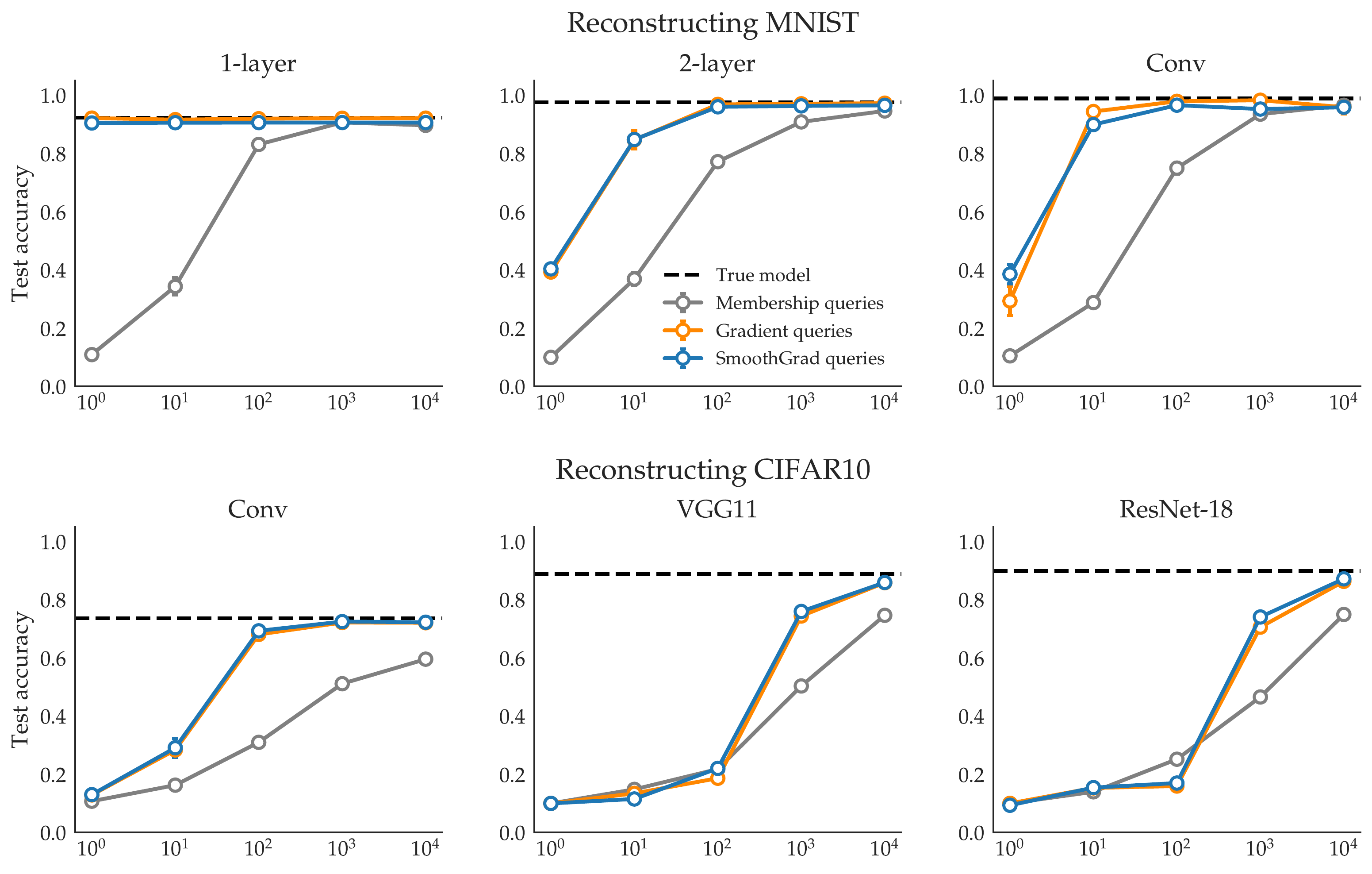}
    \caption{Access to gradients improves the accuracy of the recovered model. The improvement is approximately the same even with gradients processed by SmoothGrad.}
    \label{fig:main-results}
\end{figure}

\subsection{Main experiments: gradient queries versus membership queries} \label{sec:gqs_vs_mqs}
Figure \ref{fig:main-results} shows the results of our main experiments, described in Section \ref{sec:exps}.

\textbf{Type of query.} Across all experiments, training with gradient queries leads to orders of magnitude fewer queries required to learn the model. For example, for the MNIST convolutional model we get to 95\% accuracy in 10 gradient queries, compared to 1000 membership queries. We find practically no difference between gradient queries and SmoothGrad queries, despite picking the hyperparameters for SmoothGrad that produced the best saliency maps (See Appendix \ref{app:smoothgrad}).

\textbf{Complexity of model class.} We find that the gap in performance between gradient queries and membership queries is larger for models of lower complexity. 

As an extreme case, consider the 1-layer network on MNIST. We find a 1000x decrease in the number of queries required. With gradient queries it takes only one query to reconstruct the model (get the same performance as the original classifier). This makes sense because with gradient queries the 1-layer network is identifiable in one query, compared to 784 membership queries.\footnote{The 1-layer network is $f(x) = \sigma(w^\top x)$ where $\sigma$ is the sigmoid function and $w \in \mathbb{R}^{784}$. The model parameters $w$ are equal to $\frac{1}{f(x)(1 - f(x))}\nabla f(x)$, and thus, identifiable in one gradient and membership query.}

On MNIST with the 2-layer or convolutional network we find a 100x decrease in the number of queries needed to reconstruct the model. On CIFAR10 we find that the convolutional network (which is the same as the convolutional network used for MNIST) also has at least a 100x decrease in the number of queries needed. On the other hand, VGG11 and Resnet-18 show only a 10x decrease in the number of queries needed to reach 75\% accuracy.

\textbf{Complexity of task.} We find that the relative reduction in queries needed seems to depend on the complexity of the model class, rather than the complexity of the task. But, not surprisingly, the absolute number of queries needed increases with the complexity of the task.

On both MNIST and CIFAR10 gradient queries lead to a 100x decrease for reconstructing the convolutional network, suggesting that for the relative decrease in query complexity depends more on the complexity of the model class than the complexity of the task. However, as might be expected, for both gradient and membership queries the absolute number of queries needed increases as the complexity of the task increases. On MNIST the convolutional model is reconstructed in 10 gradient queries, compared to 1000 membership queries. On CIFAR10 the convolutional model is reconstructed in 100 gradient queries, compared to 10,000 membership queries.

\begin{figure} 
    \centering
    \includegraphics[scale=0.45]{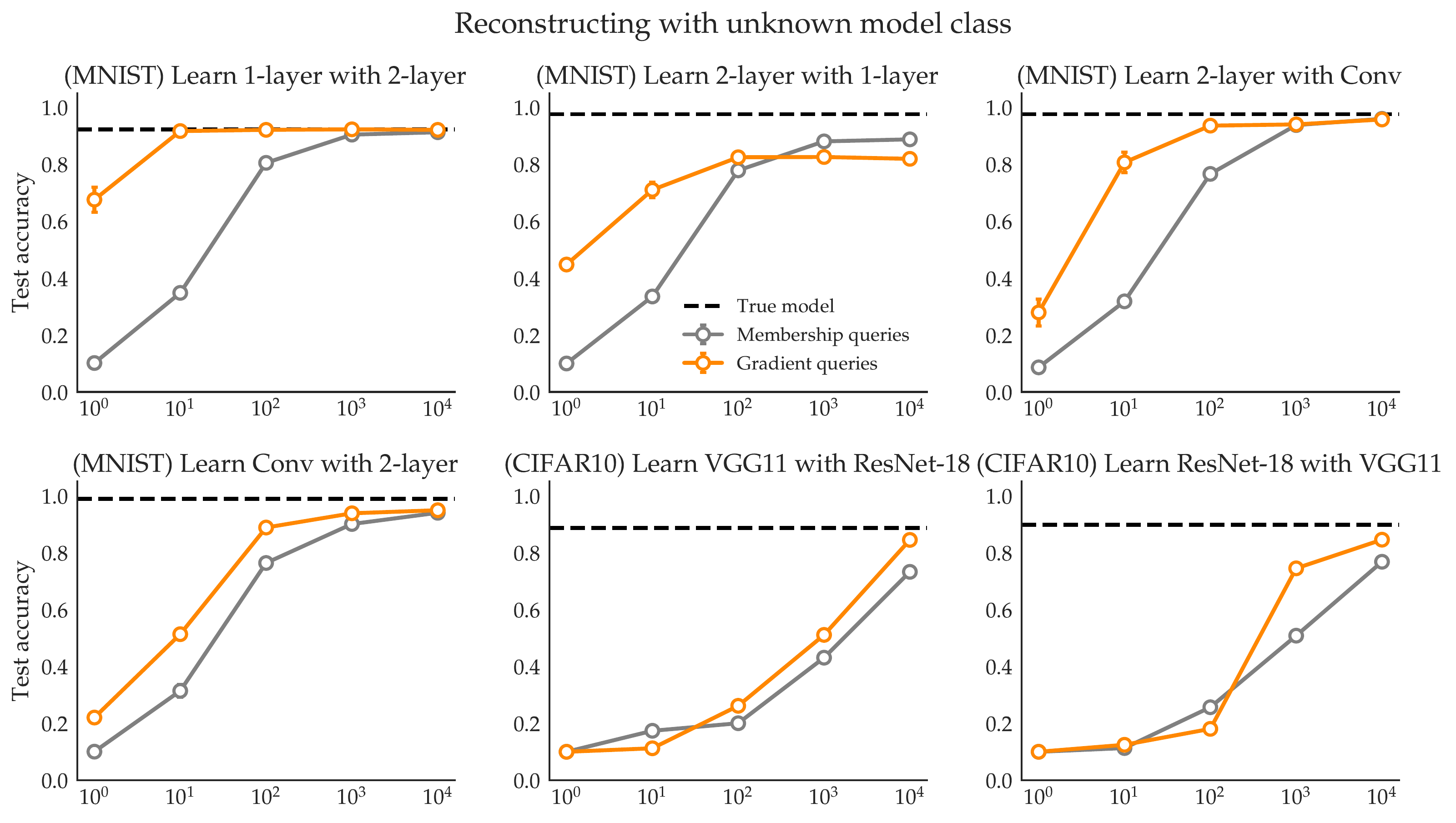}
    \setlength{\belowcaptionskip}{-15pt}
    \caption{Gradients still help when the model is unknown, but they help more when the reconstructed classifier is from a model class that is more complex than the model class of the true classifier.}
    \label{fig:diff-archs}
\end{figure}
\subsection{Unknown model class} \label{sec:unk_arch}
In the scenario where we do not know the true model class beforehand, we experiment with:
\begin{itemize}[noitemsep, nolistsep]
    \item MNIST: Reconstructing the 1-layer model with the 2-layer network (and vice versa).
    \item MNIST: Reconstructing the 2-layer model with the  convolutional network (and vice versa).
    \item CIFAR10: Reconstructing the VGG11 model with the ResNet-18 network (and vice versa).
\end{itemize}
We refer the reader to Section \ref{sec:exps} for details on the models. Figure \ref{fig:diff-archs} displays our results.

We find that gradient queries seem to help more when the the model class of $\hat{f}$ is more complex than the true classifier $f$. For example, we see a 100x decrease in the number of queries needed to reconstruct MNIST 1-layer with a 2-layer network. But, we only get an initial 10x decrease in the number of queries needed to reconstruct MNIST 2-layer with a 1-layer network. Similarly reconstructing the 2-layer network with the convolutional network works much better than reconstructing the convolutional network with the 2-layer network.

We have been fairly loose when referring to the relative complexities of different models, and it is unclear to us how to compare VGG11 and ResNet-18 in terms of complexity. Interestingly however, we find that although gradient queries still lead to a 10x decrease when reconstructing ResNet-18 with VGG11, they help very little when reconstructing a VGG11 model with a ResNet-18 network.

\begin{figure}
    \centering
    \includegraphics[scale=0.45]{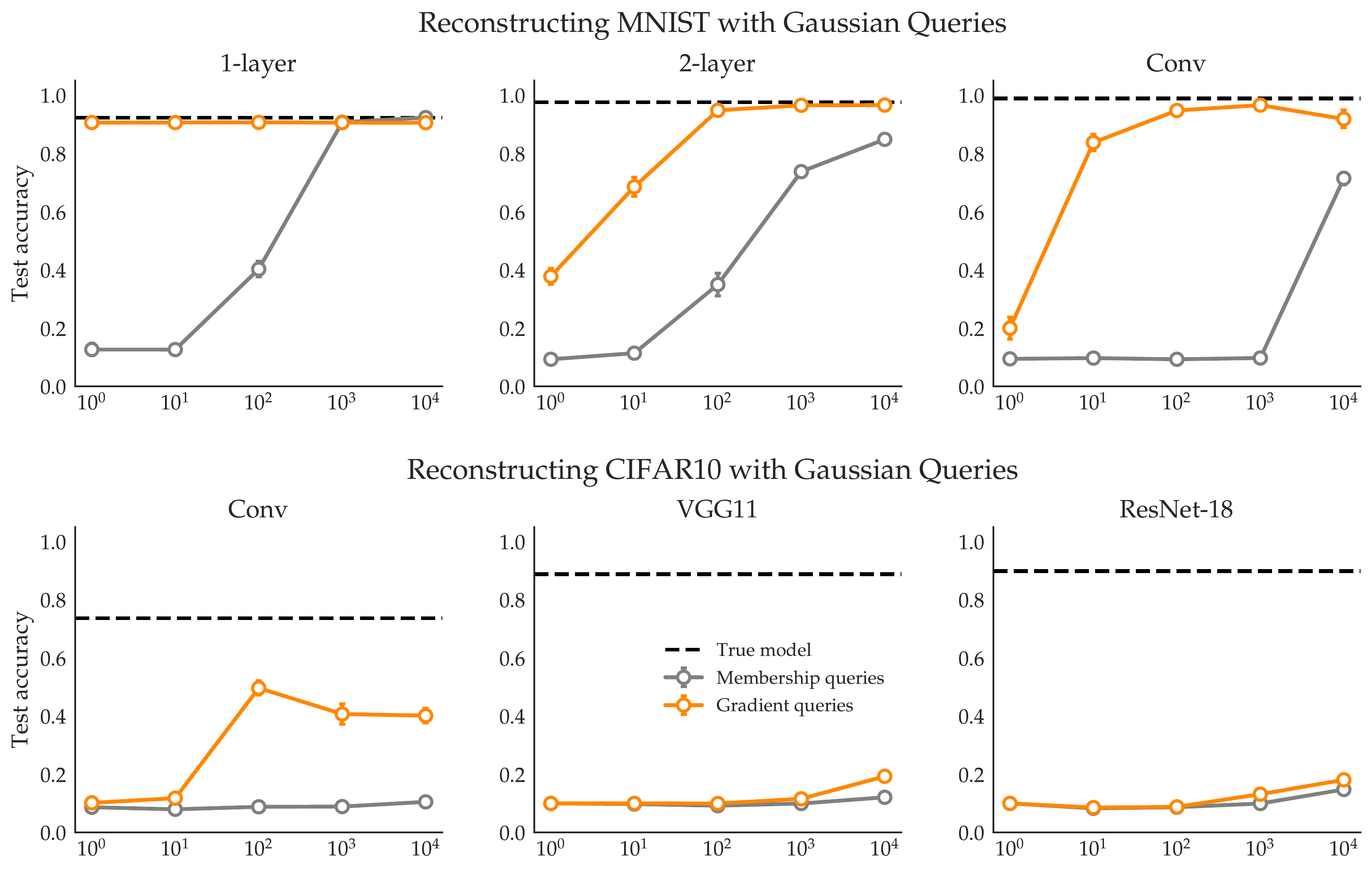}
    \caption{When querying with Gaussian generated inputs, we seem to see a \textit{larger} gap between the performance of gradient queries and the performance of membership queries.}
    \label{fig:gaussian}
\end{figure}

\subsection{Unknown data distribution} \label{sec:unk_dist}
We now analyze the setting where we do not know the data distribution. Instead we query using randomly generated Gaussian queries, i.e $x \sim \mathcal{N}(0, I_d)$. Figure \ref{fig:gaussian} displays our results. 

On MNIST we find that Gaussian queries lead to a \textit{greater} gap in performance between gradient and membership queries, compared to when using images from the data distribution.\footnote{On the 1-layer network we see the same relative decrease because it is identifiable with a single gradient + membership query or 784 membership queries, independent of the distribution the queries are generated from.} On the MNIST 2-layer network, we see at least a 1000x decrease, compared to the 100x decrease we saw in Section \ref{sec:exps} when using queries from the data distribution. On the MNIST convolutional network, we see that in 10 gradient queries we get to 84\% accuracy. On the other hand, it takes 10,000 membership queries to learn at all, and even then we get to only 71\%. Thus, we seem to get at least a 1000x decrease, compared to the 100x reduction we saw when using queries from the data distribution.

On CIFAR10 it is harder to interpret the results because the performance degrades so much for both gradient and membership queries. However, at least in the convolutional network, the gap between gradient and membership queries also seems to increase. The reconstructed model gets to 50\% accuracy in 10 gradient queries, but only to 11\% accuracy in 10,000 membership queries.

\section{Related work}
\citeauthor{tramer2016stealing} show how models can be reconstructed in practice through \textit{prediction APIs} \cite{tramer2016stealing}. Our work addresses the complementary threat of model leakage through a hypothetical \textit{explanation API}. While differential privacy can help guard against attacks from prediction APIs~\cite{dwork2018privacy}, it is not clear if this is a viable approach for preventing reconstruction from explanations. 

Learning a model via a prediction API instantiates the framework  of \textit{learning with membership queries}, in which the learner gets to actively query an oracle for labels to inputs of its choosing \citep{angluin1988queries}. In our work, we propose a complementary learning framework: \textit{learning from input gradient queries}. Similar to membership queries and prediction APIs, we believe that learning from gradients is likely to be the theoretical framework underpinning reconstruction from explanation APIs. 

We give a near-optimal algorithm for learning a two-layer network with ReLU activations through gradient queries. The geometric intuition for our algorithm is  similar to the work of \citeauthor{baum1991neural} for learning two-layer linear threshold networks with membership queries \cite{baum1991neural}.

\bibliography{main}
\bibliographystyle{plainnat}

\begin{appendices}

\section{Omitted proofs for Algorithm \ref{algo:z-recovery}} 
\label{app:step_one}
First, we prove the following two lemmas that will be useful in proving the anti-concentration and concentration bounds in Lemma \ref{lem:anti-conc-bin} and Lemma \ref{lem:conc-bin}.
\begin{lem}
\label{lem:chi2-diff}
(Anti-concentration of difference of $\chi^2_{2}$ variables) \\
Let $Q, R \sim \chi^2_{2}$. Then, $P(|Q-R| \leq \epsilon) \leq \epsilon$
for $\epsilon > 0$.
\end{lem}
\begin{proof}
Recall that the cumulative distribution function of a $\chi^{2}_{d}$ random variable $Q$ is
\begin{align*}
    P(Q \leq x) = \frac{\gamma(\frac{d}{2}, \frac{x}{2})}{\Gamma(\frac{d}{2})}
\end{align*}
where $\gamma(s, z) = \int_{0}^{z} t^{s-1}e^{-t}dt$ is the lower incomplete gamma function and $\Gamma(z) = \int_{0}^{\infty} t^{z-1}e^{-t}dt$ is the gamma function. When $d = 2$, $P(Q \leq x)$ simplifies to $\int_{0}^{x/2} e^{-z}dz$. Thus,
\begin{align*}
P(|Q-R| \leq \epsilon) & = P(R - \epsilon \leq Q \leq R + \epsilon) \\
& \leq P(0 \leq Q \leq 2\epsilon) \\
& = \int_{0}^{\epsilon}e^{-x}dx \\
& \leq \int_{0}^{\epsilon}dx \\
& = \epsilon
\end{align*}
\end{proof}
\begin{lem}
\label{lem:gauss-chi2}
(Distribution of product of independent Gaussians) \\ 
Let $X, Y \sim \mathcal{N}(0, 1)$. Then $XY$ can be written as
\begin{align*}
XY = \frac{1}{2}(Q-R)
\end{align*}
where $Q, R \sim \chi^2_{1}$ are independent.
\end{lem}
\begin{proof}
$XY$ can be rewritten as
\begin{align*}
XY = \frac{1}{4}((X+Y)^2 - (X-Y)^2).
\end{align*}
Since $\text{Cov}(X + Y, X - Y) = 0$, we know $X + Y$ and $X - Y$ are independent random variables from a $\mathcal{N}(0, 2)$ distribution. Thus, we can express $(X+Y)^2$ and $(X-Y)^2$ as $(X+Y)^2 = 2Q$ and $(X-Y)^2=2R$ where $Q, R$ are independent $\chi^{2}_{1}$ random variables. Thus, $XY = \frac{1}{2}(Q-R)$.
\end{proof}

\begin{replemma}{lem:anti-conc-bin}
Let $a, b \in \mathcal{S}^{d-1}$ be unit vectors such that $|\langle a, b \rangle| \leq 1 - c$ for some scalar $c \in [0, 1]$. Suppose we pick random vectors $u, v \sim \mathcal{N}(0, I_d)$. Let $t_1, t_2 \in \mathbb{R}$ be scalars such that $\langle a, u + t_1 v \rangle = 0$ and $\langle b, u + t_2 v \rangle = 0$.\footnote{With probability one such a $t$ exists.} Then,
\begin{align*}
   P(|t_1 - t_2| \leq \epsilon) \leq 3^{\frac{4}{3}}\left (\frac{\epsilon}{c} \right )^{\frac{2}{3}}.
\end{align*}
\end{replemma}
\begin{proof}
Solving for the scalars $t_1$ and $t_2$ yields
\begin{align*}
    t_1 = -\frac{\langle a, u \rangle}{\langle a, v \rangle}, ~t_2 = -\frac{\langle b, u \rangle}{\langle b, v \rangle}.
\end{align*}
Let $a^{\perp}$ be a unit vector orthogonal to $a$. The vector $b$ can be expressed as $b = \beta_1 a + \beta_2 a^{\perp}$ where $\beta_1, \beta_2 \in \mathbb{R}$ and $\beta_1 = \langle a, b \rangle$. The coefficient $\beta_2$ can be lower bounded as the following.
\begin{align*}
    \beta_2 = \sqrt{1 - \beta_1^2} \geq \sqrt{2c - c^2} \geq c
\end{align*}
Using this expression for the vector $b$ we can rewrite $|t_1 - t_2|$ as
\begin{align}
    \label{eq:thres}
    |t_1 - t_2| & = \left | \frac{\langle a, u \rangle}{\langle a, v \rangle} - \frac{\langle b, u \rangle}{\langle b, v \rangle} \right | \nonumber \\
    & = \frac{|\beta_2(\langle a, u \rangle \langle a^{\perp}, v \rangle - \langle a, v \rangle \langle a^{\perp}, u \rangle)|}{|\langle a, v \rangle (\beta_1 \langle a, v \rangle + \beta_2 \langle a^{\perp}, v \rangle )|} \nonumber \\
    & = \frac{|\beta_2(X_1Y_2 - X_2Y_1)|}{|X_2(\beta_1 X_2 + \beta_2 Y_2)|},
\end{align}
where $X_1 = \langle a, u \rangle, X_2 = \langle a, v \rangle, Y_1 = \langle a^{\perp}, u \rangle, Y_2 = \langle a^{\perp}, v \rangle$ are independent $\mathcal{N}(0, 1)$ random variables. To bound $P(|t_1 - t_2| \leq \epsilon)$ we can bound the numerator and denominator of (\ref{eq:thres}) separately. For all $k > 0$, the following inequality holds.
\begin{align*}
P(|t_1 - t_2| \geq \epsilon) \geq P \left (|\beta_2(X_1Y_2 - X_2Y_1)| \geq k \sqrt{\epsilon}, \,|X_2(\beta_1 X_2 + \beta_2 Y_2)| \leq \frac{k}{\sqrt{\epsilon}} \right )
\end{align*}
Applying a union bound to the complementary event yields,
\begin{equation}
\label{eq:union}
P(|t_1 - t_2| \leq \epsilon) \leq P(|\beta_2(X_1Y_2 - X_2Y_1)| \leq k \sqrt{\epsilon}) + P \left ( |X_2(\beta_1 X_2 + \beta_2 Y_2)| \geq \frac{k}{\sqrt{\epsilon}} \right ).
\end{equation}
Applying Lemma \ref{lem:gauss-chi2} to the independent products $X_1Y_2$ and $X_2Y_1$ simplifies the numerator to
\begin{align*}
    |\beta_2(X_1Y_2-X_2Y_1)| = \left | \frac{\beta_2}{2}(Q - R) \right|,
\end{align*}
where $Q, R \sim \chi^{2}_2$ are independent Chi-squared random variables. Then by Lemma \ref{lem:chi2-diff},
\begin{align*}
    P(|\beta_2(X_1Y_2 - X_2Y_1)| \leq k \sqrt{\epsilon}) \leq P \left ( |Q - R| \leq \frac{2k \sqrt{\epsilon}}{\beta_2} \right ) \leq \frac{2k \sqrt{\epsilon}}{\beta_2}.
\end{align*}
To upper bound the tail probability of the denominator (the second term in Equation \ref{eq:union}) note that
\begin{align*}
\mathbb{E}[(X_2(\beta_1 X_2 + \beta_2 Y_2))^2] = 3\beta_1^2 + \beta_2^2.
\end{align*}
Then by Markov's Inequality,
\begin{align*}
    P \left ( |X_2(\beta_1 X_2 + \beta_2 Y_2)| \geq \frac{k}{\sqrt{\epsilon}} \right ) =  P \left ( (X_2(\beta_1 X_2 + \beta_2 Y_2))^2 \geq \frac{k^2}{\epsilon} \right )\leq \frac{(3\beta_1^2 + \beta_2^2) \epsilon}{k^2} = \frac{3 \epsilon}{k^2}.
\end{align*}
Therefore,
\begin{align*}
    P(|t_1 - t_2| \leq \epsilon) \leq \frac{2k \sqrt{\epsilon}}{\beta_2} + \frac{3\epsilon}{k^2}.
\end{align*}
Minimizing the right-hand side with respect to $k$ yields
\begin{align*}
P(|t_1 - t_2| \leq \epsilon) \leq 3^{\frac{4}{3}} \left ( \frac{\epsilon}{\beta_2} \right )^{\frac{2}{3}} \leq  3^{\frac{4}{3}} \left ( \frac{\epsilon}{c} \right )^{\frac{2}{3}} \leq O \left ( \left ( \frac{\epsilon}{c} \right )^{\frac{2}{3}} \right ).
\end{align*}
\end{proof}

\begin{replemma}{lem:conc-bin}
Let $a \in \mathcal{S}^{d-1}$ be a unit vector. Suppose we pick random vectors $u, v \sim \mathcal{N}(0, 1)$. Let $t \in \mathbb{R}$ be the value such that $\langle a, u + t v \rangle = 0$. Then,
\begin{align*}
   P(|t| \geq l) \leq \frac{2}{\pi l}.
\end{align*}
\end{replemma}
\begin{proof}
$t = -\frac{\langle a, u \rangle}{\langle a, v \rangle}$ follows a standard Cauchy distribution. The cumulative distribution function of a standard Cauchy random variable $X$ is $P(X \leq a) = \frac{1}{\pi}\arctan(a) + \frac{1}{2}$. Thus, 
\begin{align*}
    P(|t| \geq l) & = P(t \leq -l) + P (t \geq l) \\
    & = \left ( \frac{1}{\pi}\arctan(-l) + \frac{1}{2} \right) + \left ( \frac{1}{2} - \frac{1}{\pi}\arctan(l) \right ) \\
    & = 1 - \frac{2}{\pi}\arctan(l) \\
    & \leq 1 - \frac{2}{\pi} \left ( \frac{\pi}{2} - \frac{1}{l} \right ) \\
    & = \frac{2}{\pi l}
\end{align*}
\end{proof}

\section{Picking query points in Algorithm \ref{algo:s-recovery}} \label{app:step_two}
\label{sec:pick-X}
For completeness, we show that we can easily find a matrix $X$ which satisfy the requirements of Lemma \ref{lem:fullrank} through the following steps:
\begin{enumerate}
    \item Pick a random $v \in \mathbb{R}^{d}$. Let $g(v) = 1\{Zv \geq 0\}$ and $\mathcal{C} = \{x \mid g(v) = g(x) \}$ be the cell containing $v$.
    \item Find the center $y_0 \in \mathbb{R}^d$ and radius $r \in \mathbb{R}$ of the largest $\ell_2$ ball  within $\mathcal{C} \cap [0, 1]^{d}$. The center of this ball is known as the \textit{Chebyshev center} of $\mathcal{C} \cap [0, 1]^{d}$. It is well known that the center $y_0$ and radius $r$ can be solved for through a linear program \citep{boyd2004convex}. We include the linear program for our case below.
    \begin{align*}
        & \max_{y_0, r} r \\
        & \text{subject to} \\
        & (z_i^\top y_0 + r)\sgn(z_i^\top v) \geq 0 \\
        & \mathbf{0} \leq y_0 \leq \mathbf{1}
    \end{align*}
    \item Construct a set of $d$ linearly independent vectors $\mathcal{Y} = y_1, \dots, y_d \in \mathcal{C}$ as follows.
    \begin{align*}
        & y_i = y_0 + \Delta^i \\
        & \Delta^i_j = \begin{cases}
            r/2 & i = j \\
            0 & i \neq j
        \end{cases}
    \end{align*}
    \item Let $Y \in \mathbb{R}^{d \times d}$ be the matrix whose columns are formed by the vectors in $\mathcal{Y}$. Since $Y$ has rank $d$, $\text{Rank}(ZY) = h$. Thus, we can pick $h$ vectors $x_1, \dots, x_h$ from $\mathcal{Y}$ such that $\text{Rank}(ZX) = h$ where $X$ is the matrix whose columns are the vectors $x_1, \dots, x_h$.
\end{enumerate}

\section{Reconstruction from membership queries} \label{app:membership}
We now consider how to reconstruct the two-layer ReLU neural network described in Section \ref{sec:model} with membership queries alone, rather than membership and gradient queries. We show that we can convert our algorithm into one that learns with membership queries by estimating the gradients of $f$ with membership queries.

\subsection{Membership query version of Algorithm \ref{algo:model-recovery}} \label{sec:mq-algo}
We define the membership query version of Algorithm \ref{algo:model-recovery}, referred to as Algorithm \ref{algo:model-recovery}-MQ, by replacing any use of the gradient $\nabla f(x)$ with an estimate of the gradient, $\widehat{\nabla} f(x)$, computed with $d$ membership queries. We estimate the gradient by estimating each component separately through a finite difference approximation:
\begin{align*}
    \widehat{\nabla} f(x)_j = \frac{f(x+\Delta^j) - f(x)}{s} \,,
\end{align*}
where $i, j \in [d]$, $s \in \mathbb{R}^{+}$ and $
    \Delta^j_i = \begin{cases}
        s & \text{if } i = j \\
        0 & \text{o.w.}
    \end{cases}$. \\

\noindent Our main result shows that we can recover the function $f$ in $O(dh \log \frac{h}{\delta})$ membership queries:
\begin{thm} \label{thm:mq}
With probability $1-\delta$, if $s \leq \frac{\delta \epsilon}{2(2-\delta)l \epsilon}$, where $l, \epsilon \in \mathbb{R}$ are parameters of the binary search in Algorithm \ref{algo:z-recovery}, then Algorithm \ref{algo:model-recovery}-MQ returns a function $\hat{f}$ such that $\hat{f} = f$ in $O(dh \log \frac{h}{\delta})$ membership queries.
\end{thm}

\noindent The next subsection contains our proofs.

\subsection{Proofs}

The proof of Theorem \ref{thm:mq} relies on showing that we can pick an $s$ small enough so that with high probability all estimates of the gradient are equal to the exact gradient. We show this by proving that if all points used in estimating a gradient lie in the same cell defined by the separating hyperplanes of $f$, then the estimate of the gradient $\widehat{\nabla} f(x)$ is equal to the gradient $\nabla f(x)$. If $s$ is small enough, then all points evaluated for a gradient estimate will lie in the same cell, and thus the exact gradient will be recovered. By choosing $s$ small enough, we can ensure that all gradients estimated by Algorithm \ref{algo:model-recovery}-MQ are equal to the exact gradient with high probability.


First, we show that if all points sampled in estimating the gradient lie in the same cell, then the estimate of the gradient $\widehat{\nabla} f(x)$ is equal to the gradient $\nabla f(x)$:
\begin{lem} \label{lem:same-cell-same-grad}
    Suppose for all $j \in [d]$, $x+\Delta^j$ lies in the same cell as $x$, i.e,
    \begin{align*}
        \mathbb{I}\{Ax \geq 0\} =  \mathbb{I}\{A(x+\Delta^j) \geq 0\} \,.
    \end{align*}
    Then, $\widehat{\nabla} f(x) = \nabla f(x)$.
\end{lem}
\begin{proof}
Recall that the function $f$ can be expressed as
\begin{align*}
    f(x) = w^\top \text{Diag}(\mathbb{I}\{Ax \geq 0 \})Ax \,.
\end{align*}

Thus, the $j$-th component of the gradient of $f$ is
\begin{align*}
    \nabla f(x)_j = w^\top \text{Diag}(\mathbb{I}\{Ax \geq 0 \})a_j \,,
\end{align*}

where $a_j$ is the $j$-th column of $A$. Our estimate of the gradient is
\begin{align*}
    \widehat{\nabla} f(x)_j & = \frac{f(x + \Delta^j) - f(x)}{s} \\
    & = \frac{w^\top \text{Diag}(\mathbb{I}\{A(x + \Delta^j) \geq 0 \})A(x + \Delta^j) - w^\top \text{Diag}(\mathbb{I}\{Ax \geq 0 \})Ax}{s} \\
    & = \frac{w^\top \text{Diag}(\mathbb{I}\{Ax \geq 0 \})A(x + \Delta^j) - w^\top \text{Diag}(\mathbb{I}\{Ax \geq 0 \})Ax}{s} \\
    & = \frac{w^\top \text{Diag}(\mathbb{I}\{Ax \geq 0 \})A\Delta^j}{s} \\
    & = \frac{w^\top \text{Diag}(\mathbb{I}\{Ax \geq 0 \})a_js}{s} \\
    & = w^\top \text{Diag}(\mathbb{I}\{Ax \geq 0 \})a_j \\
    & = \nabla f(x)_j \, .
\end{align*}

Therefore, $\widehat{\nabla} f(x) = \nabla f(x)$.
\end{proof}

The next lemma shows that if $s$ is small enough, then all points evaluated used to estimate a gradient lie in the same cell, and thus the exact gradient is recovered.

\begin{lem}
\label{lem:small-s-same-grad}
Suppose $s \in \mathbb{R}^{+}$ is such that $|Ax| \geq s\mathbf{1}$, then $\widehat{\nabla} f(x) = \nabla f(x)$.
\end{lem}
\begin{proof}
We simply need to prove that $\mathbb{I} \{A(x + \Delta^{j}) \geq 0 \} = \mathbb{I} \{Ax \geq 0 \}$ for all $j \in [d]$ and then the result follows by Lemma \ref{lem:same-cell-same-grad}. Since the rows of the weight matrix $A$ are unit norm, we know $|sa_j| \leq |s\mathbf{1}| \leq |Ax|$ where $a_j$ is the $j$-th column of $A$. Thus, $\mathbb{I} \{A(x + \Delta^{j}) \geq 0 \} = \mathbb{I} \{Ax + sa_j \geq 0\} = \mathbb{I}\{Ax \geq 0 \}$. The result then follows from Lemma \ref{lem:same-cell-same-grad}.
\end{proof}

Next, given a particular value of $s$, we bound the probability that all gradients we estimate with our algorithm are exactly equal to the true gradient.
 
\begin{lem} 
\label{lem:exact-grads-in-algo} Let $\mathcal{X} = \{u + i \epsilon v \mid |i| \leq l/\epsilon, i \in \mathbb{Z} \}$ be the set of points Algorithm \ref{algo:z-recovery} may query. Then,
\begin{align*}
    P(\forall x \in \mathcal{X} ~~\widehat{\nabla} f(x) = \nabla f(x)) \geq 1 - \frac{2lhs}{\epsilon} \,.
\end{align*}
\end{lem}
\begin{proof}
First we will establish a bound for one row $a$ of the weight matrix $A$.
\begin{align*}
& P(\exists x \in \mathcal{X} : |\langle a, x \rangle| \leq s) \\
& \leq \sum_{x \in \mathcal{X}} P(|\langle a, x \rangle| \leq s) & \text{Union bound} \\
& = \sum_{|i| \leq l/\epsilon, i \in \mathbb{Z}} P(|\langle a, u + i \epsilon v \rangle| \leq s) \\
& = \sum_{|i| \leq l/\epsilon, i \in \mathbb{Z}} P(|Z_i| \leq s) & Z_i \sim \mathcal{N}(0, 1 + (\epsilon i)^2) \\
& \leq \frac{2l}{\epsilon} P(|Z| \leq s) & Z \sim \mathcal{N}(0, 1) \\
& \leq \frac{2ls}{\epsilon} & \text{Gaussian anti-concentration}
\end{align*}

A union bound on all rows of the weight matrix $A$ then shows that
\begin{align*}
    P(\exists x \in \mathcal{X}, i \in [h] : |\langle A_i, x \rangle | \leq s)
    \leq \sum_{i=1}^{h} P(\exists x \in \mathcal{X}: |\langle A_i, x \rangle | \leq s)
    \leq \frac{2lhs}{\epsilon} \,.
\end{align*}

Thus, by Lemma \ref{lem:small-s-same-grad},
\begin{align*}
    P(\forall x \in \mathcal{X} ~~\widehat{\nabla} f(x) = \nabla f(x)) = P(\forall x \in \mathcal{X} ~~|Ax| \geq s \mathbf{1}) \geq 1 - \frac{2ls}{\epsilon} \,.
\end{align*}
\end{proof}

Finally, we show that by picking $s$ small enough so that all gradients estimate are exact with high probability, the sample complexity of Algorithm \ref{algo:model-recovery}-MQ becomes $O(dh \log \frac{h}{\delta})$ membership queries.
\begin{reptheorem}{thm:mq}
With probability $1-\delta$, if $s \leq \frac{\delta \epsilon}{2(2-\delta)l \epsilon}$, where $l, \epsilon \in \mathbb{R}$ are parameters of the binary search in Algorithm \ref{algo:z-recovery}, then Algorithm \ref{algo:model-recovery}-MQ returns a function $\hat{f}$ such that $\hat{f} = f$ in $O(dh \log \frac{h}{\delta})$ membership queries.
\end{reptheorem}
\begin{proof}
Algorithm \ref{algo:model-recovery} only uses gradients of $f$ in Algorithm \ref{algo:z-recovery} and Algorithm \ref{algo:model-recovery} succeeds if and only if Algorithm \ref{algo:z-recovery} succeeds. Thus, we can bound the success of Algorithm \ref{algo:model-recovery}-MQ by bounding the probability that all gradients used in Algorithm \ref{algo:z-recovery} are estimated exactly.

In $O(h \log \frac{2h}{\delta}) = O(h \log \frac{h}{\delta})$ gradient queries we can guarantee that Algorithm \ref{algo:z-recovery} succeeds with probability $1 - \frac{\delta}{2}$. The probability Algorithm \ref{algo:model-recovery}-MQ succeeds then becomes the following.
\begin{align*}
    & P(\text{Algorithm \ref{algo:model-recovery}-MQ succeeds}) \\
    & = P(\text{Algorithm \ref{algo:z-recovery} succeeds} | \text{Exact gradients}) P(\text{Exact gradients}) \\
    & \geq \left ( 1 - \frac{\delta}{2} \right) \left ( 1 - \frac{2lh}{\epsilon}{s} \right) & (\text{Lemma \ref{lem:exact-grads-in-algo}}) \\ 
    & = \left ( 1 - \frac{\delta}{2} \right) \left ( 1 - \frac{2lh}{\epsilon} \left ( \frac{\delta \epsilon}{2(2-\delta)lh} \right ) \right) \\
    & \geq 1 - \delta
\end{align*}

Since, it takes $d$ membership queries to compute each gradient that Algorithm \ref{algo:z-recovery} requires, the sample complexity becomes $O(dh \log \frac{h}{\delta})$ membership queries.

\end{proof}

\section{SmoothGrad} \label{app:smoothgrad}

Instead of returning the raw gradient $\nabla f(x)$, SmoothGrad \citep{smilkov2017smoothgrad} returns an average of gradients around the input $x$:
\begin{align*}
\widetilde{\nabla} f(x) = \sum_{i=1}^{N} \frac{1}{N} \nabla f(x + z_i) ,\, 
\end{align*}
where $z_i \sim \mathcal{N}(0, \sigma^2 I)$ and $N > 0$. SmoothGrad has two hyperparameters: (1) $\sigma$ the standard deviation of the Gaussian noise and (2) $N$ the number of samples to pick. 

As shown in Figure \ref{fig:mnist-saliency}, we found that the best value of $\sigma$ for MNIST was 1000 times $\sigma_D$, the standard deviation of the images in the dataset. On CIFAR10 using either the VGG-11 or ResNet-18 network, no value of $\sigma$ seems to produce a sharp map (Figures \ref{fig:cifar-vgg-saliency} and \ref{fig:cifar-resnet-saliency}). So for our CIFAR10 experiments, we set $\sigma$ equal to the standard deviation of the dataset $\sigma_D$. In the original SmoothGrad paper, \citeauthor{smilkov2017smoothgrad} find that the best value of $\sigma$ for MNIST is about 70\% the spread of the dataset, while on ImageNet it is only 10-20\%. So the difference between the value of $\sigma$ we use on MNIST and the value of $\sigma$ we use on CIFAR10 seems to qualitatively match the difference in the value of $\sigma$ \citeauthor{smilkov2017smoothgrad} use on MNIST and ImageNet.

We expect that SmoothGrad may eventually degrade the performance of the reconstructed model as $\sigma$ increases. But at least for the values of $\sigma$ we test, which are already quite large relative to the standard deviation of the dataset, and seem to match values that may be used in practice, we see no degradation in performance when using gradients preprocessed by SmoothGrad.

Regarding the number of samples, $N$, \citeauthor{smilkov2017smoothgrad} state that the estimated gradient becomes smoother as $N$ increases, but that they find diminishing returns for $N > 50$. For computational reasons we set $N=10$ in our experiments, however, this should only make it \textit{harder} to learn, since the outputs of SmoothGrad become noisier.

\begin{figure}
    \centering
    \includegraphics[scale=0.7]{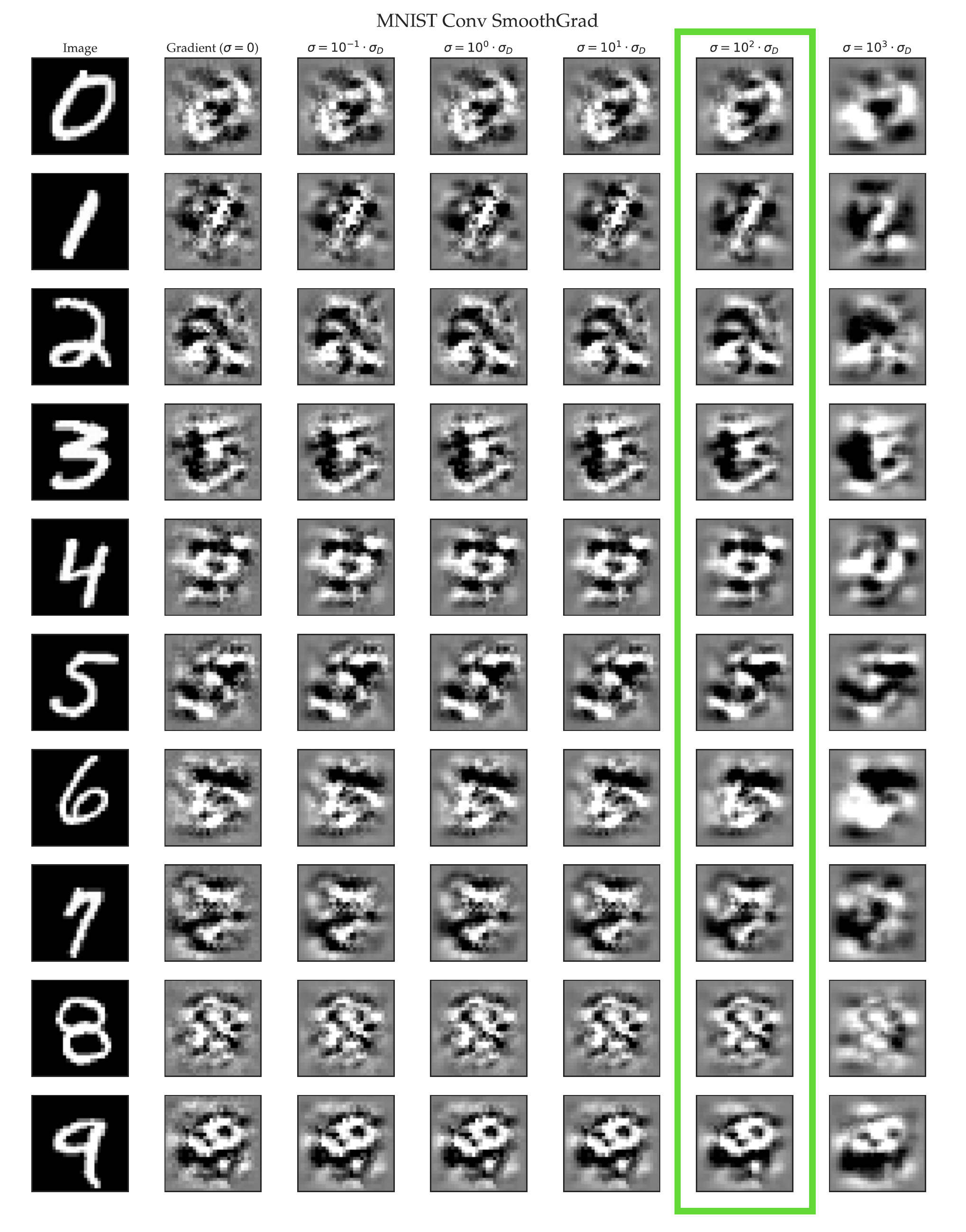}
    \caption{Saliency maps computed with SmoothGrad on the MNIST convolutional network described in Section \ref{sec:exps} using $N=100$. For our experiments in Section \ref{sec:exps}, we choose $\sigma$ to be 1000 times the standard deviation $\sigma_D$ of the dataset (highlighted column).}
    \label{fig:mnist-saliency}
\end{figure}

\begin{figure}
    \centering
    \includegraphics[scale=0.7]{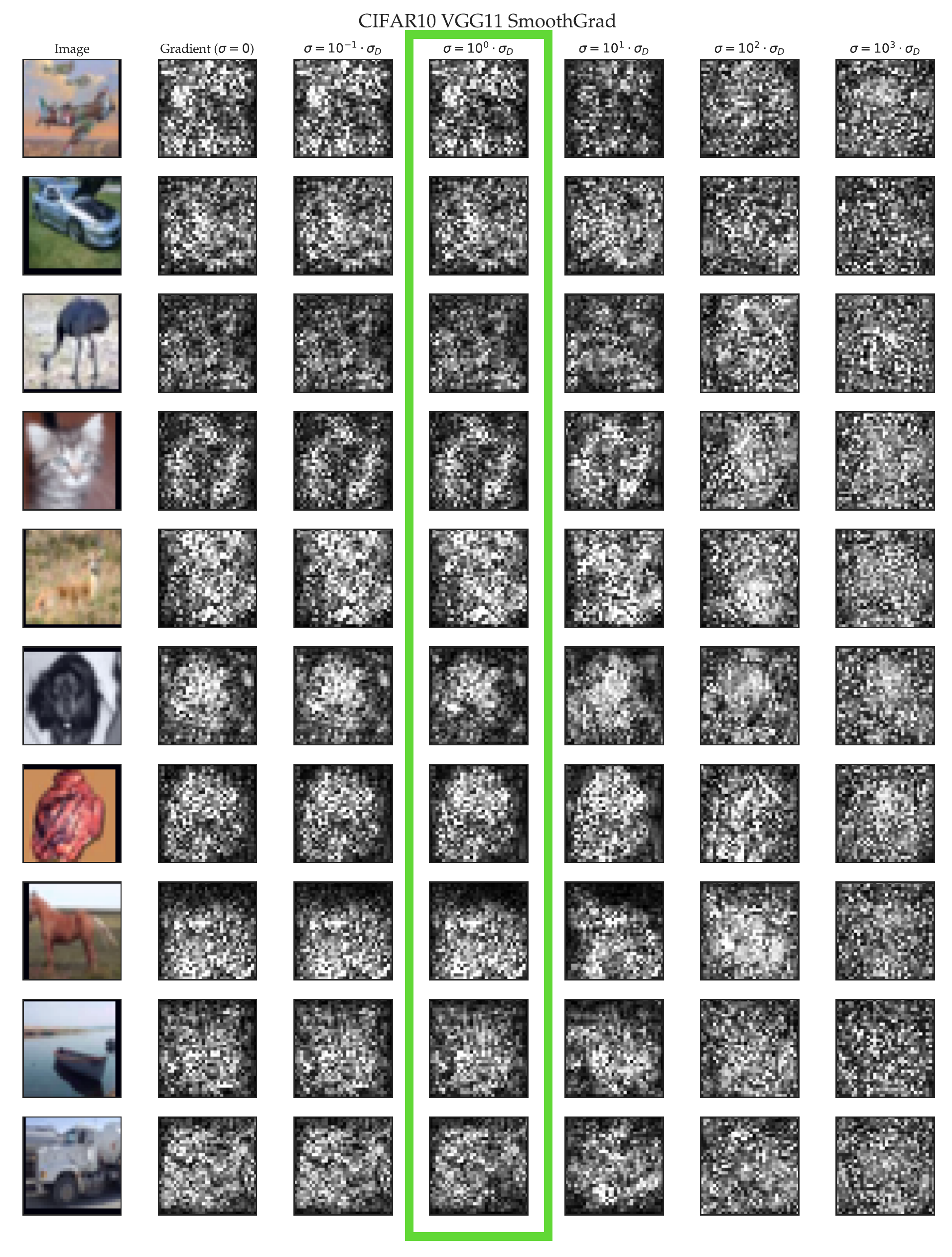}
    \caption{Saliency maps computed with SmoothGrad on the CIFAR10 VGG-11 network using $N=100$. Following \cite{smilkov2017smoothgrad}, for CIFAR10, which has RGB images, we visualize the absolute value of the output of SmoothGrad. For our experiments in Section \ref{sec:exps}, we choose $\sigma$ to be equal to the standard deviation $\sigma_D$ of the dataset (highlighted column).}
    \label{fig:cifar-vgg-saliency}
\end{figure}

\begin{figure}
    \centering
    \includegraphics[scale=0.7]{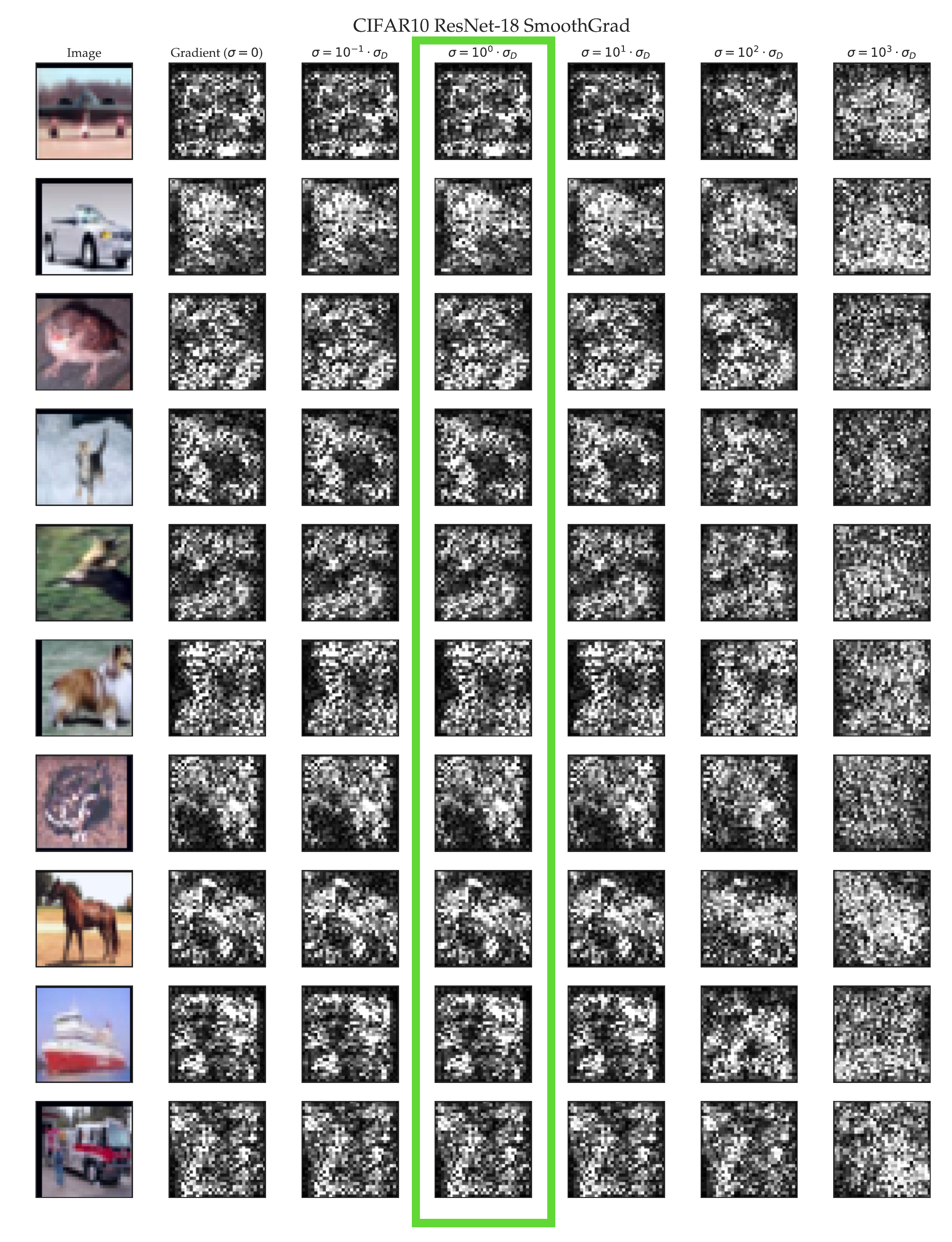}
    \caption{Saliency maps computed with SmoothGrad on the CIFAR10 ResNet-18 network using $N=100$. Following \cite{smilkov2017smoothgrad}, for CIFAR10, which has RGB images, we visualize the absolute value of the output of SmoothGrad. For our experiments in Section \ref{sec:exps}, we choose $\sigma$ to be equal to the standard deviation $\sigma_D$ of the dataset (highlighted column).}
    \label{fig:cifar-resnet-saliency}
\end{figure}
\end{appendices}

\end{document}